\documentclass{llncs}
\usepackage{hyperref}
\usepackage{amssymb}
\usepackage{verbatim}
\usepackage{multirow}
\usepackage{rotating}
\usepackage{hyperref}

\usepackage{graphicx}

\renewenvironment{proof}{\medskip\par\noindent{\sc Proof:}}{\nobreak\hfill\rule{2mm}{2mm}\medskip\par}

\begin{document}
\sloppy
\pagestyle{plain}

\title{The Fractal Dimension of SAT Formulas}

\author{Carlos Ans\'otegui\inst{1} \and
Maria Luisa Bonet\inst{2} \and Jes\'us Gir\'aldez-Cru\inst{3} \and
Jordi Levy\inst{3}}

\institute{
DIEI,
Univ. de Lleida,
\email{carlos@diei.udl.cat}
\and
LSI, UPC,
\email{bonet@lsi.upc.edu}
\and
IIIA-CSIC,
\email{jgiraldez,levy@iiia.csic.es}}

\maketitle

\begin{abstract}
Modern SAT solvers have experienced a remarkable progress on solving industrial instances. Most of the techniques have been
developed after an intensive experimental testing process. Recently,
there have been some attempts to analyze the structure of these
formulas in terms of complex networks, with the long-term aim of
explaining the success of these SAT solving techniques, and possibly 
improving them.

We study the fractal dimension of SAT formulas, and show that most
industrial families of formulas are self-similar, with a small fractal
dimension.  We also show that this dimension is not affected by the
addition of learnt clauses.  We explore how the dimension of a
formula, together with other graph properties can be used to
characterize SAT instances. Finally, we give empirical evidence that these graph
properties can be used in state-of-the-art portfolios.
\end{abstract}

\bibliographystyle{abbrv}

\abovedisplayskip 0.2\abovedisplayskip
\belowdisplayskip 0.2\belowdisplayskip

\section{Introduction}

The SAT community has been able to come up with successful SAT solvers
for industrial applications.  However, nowadays we can hardly explain
why these solvers are so efficient working on industrial SAT instances
with hundreds of thousands of variables and not on random instances
with hundreds of variables. The common wisdom is that the success of
modern SAT/CSP solvers is correlated to their ability to exploit the
hidden structure of real-world instances~\cite{WilliamsGS03}.
Unfortunately, there is no precise definition of the notion of
structure.
 
At the same time, the community of complex networks has produced tools for
describing and analyzing the structure of social, biological and
communication networks~\cite{barabasi99} which can explain some
interactions in the real-world.

Representing SAT instances as graphs, we can use some of the
techniques from complex networks to characterize the structure of SAT
instances. Recently, some progress has been made in this direction.
It is known that many industrial instances have the \emph{small-world}
property~\cite{Walsh99}, exhibit high \emph{modularity}~\cite{SAT12},
and have a \emph{scale-free structure}~\cite{CP09}.
In~\cite{pagerank}, the \emph{eigenvector centrality} of variables in
industrial instances is analyzed. It is shown that it is correlated
with some aspects of SAT solvers. For instance, decision variables
selected by the SAT solvers, are usually the most central variables in
the formula.  However, how these analysis may help improve the
performance of SAT solvers is not known at this stage.  The
\emph{fractal structure} of search spaces, and its relation with the
performance of randomized search methods, is studied
in~\cite{Selman99}

The first contribution of this paper is to analyze the existence of
self-similarity in industrial SAT instances. The existence of a
self-similar structure would mean that after \emph{rescaling}
(replacing groups of nodes at a given distance by a single node, for
example), we would observe the same kind of structure. It would also
mean that the diameter $d^{max}$ of the graph grows as $d^{max} \sim
n^{1/d}$, where $d$ is the fractal dimension of the graph, and not as
$d^{max} \sim \log n$, as in random graphs or small-world
graphs. Therefore, actions in some part of the graph (like variable
instantiation) would not \emph{propagate} to other parts as fast as in
random graphs. Our analysis shows that most industrial formulas are
self-similar. Also fractal dimension does not change much during the
execution of modern SAT solvers.

Studying graph properties of formulas has several direct applications. 
One of them, is the generation of industrial-like random SAT
instances.  Understanding the structure of industrial instances is a
first step towards the development of random instance generators,
reproducing the features of industrial instances. Related work in this
direction can be found in~\cite{ijcai09}.

Another potential application, is to improve portfolio
approaches~\cite{Xu:2008,Kadioglu:2011} which are solutions to the
algorithm selection problem~\cite{Rice76}. State-of-the-art SAT
Portfolios compute a set of features of SAT instances in order to
select the best solver from a predefined set to be run on a particular
SAT instance. It is reasonable to think that more informative
structural features of SAT instances can help to improve portfolios.

The second contribution of this paper is an experimental study that
shows how to use graph properties plus the clause/variable ratio in
modern state-of-the-art portfolios. The graph properties we use are:
the distribution of variable frequencies, the modularity and the
fractal dimension of a SAT formula. We show that using this reduced
set of properties we are able to classify instances into families
slightly better than the portfolio SATzilla2012~\cite{SATzilla2012},
which currently uses a total of $138$ features. Secondly, we show that
these features could be used as the basis of a portfolio SAT solver,
showing that they give a level of information similar to all SATzilla
features together. Let us emphasize that the fractal dimension is
crucial in obtaining these results.

The paper proceeds as follows. We introduce the fractal dimension of
graphs in Section~\ref{sec:dimgraph}. Then, we analyze whether SAT
instances represented as graphs do have a fractal dimension in
Section~\ref{sec:dimSAT}, and the effect of learnt clauses.  In
Section~\ref{sec:other}, we describe two additional previously studied
graph features of SAT instances, the $\alpha$ exponent and the
modularity.  Section~\ref{sec:port} describes briefly portfolios
approaches and the set of features currently used.  We finish in
Section~\ref{sec:exp} presenting some experimental results on the
feature-based classification of SAT instances, and conclude in
Section~\ref{sec:conclusion}.  We also include an appendix with the
numeric values used in some of the figures.

\section{Fractal Dimension of a Graph}\label{sec:dimgraph}

We can define a notion of fractal dimension of a graph following the
principle of self-similarity. We will use the definition of box
covering by Hausdorff~\cite{Mandelbrot}.

\begin{definition}\parindent 0mm
  Given a graph $G$, a {\bf box} $B$ of size $l$ is a subset of nodes
  such that the distance between any pair of them is smaller than $l$.

  Let $N(l)$ be the minimum number of boxes of size $l$ required to
  \emph{cover} the graph. We say that a set of boxes covers a graph,
  if every node of the graph is in some box.

  We say that a graph has the {\bf self-similarity} property if the
  function $N(l)$ decreases polynomially, i.e. $N(l) \sim l^{-d}$, for
  some value $d$.

  In this case, we call $d$ the {\bf dimension} of the graph.
\end{definition}
Notice that $N(1)$ is equal to the number of nodes of $G$, and
$N(d^{max}+1)=1$ where $d^{max}$ is the diameter of the graph.

\begin{lemma}
  Computing the function $N(l)$ is NP-hard.\footnote{
    In~\cite{SGHM07} the same result is stated, but there, they
    prove the wrong reduction. They reduce the computation of $N(2)$ to 
    the graph coloring problem.}
\end{lemma}
\begin{proof}
  We can reduce the graph coloring problem to the computation of
  $N(2)$ as follows.  Given a graph $G$, let $\overline{G}$, the
  complement of $G$, be a graph with the same nodes, and where any
  pair of distinct nodes are connected in $\overline{G}$ iff they are
  not connected in $G$. Boxes of size $2$ in $\overline{G}$ are
  cliques, thus they are sets of nodes of $G$ without an edge between
  them. Therefore, the minimal number of colors needed to color $G$ is
  equal to the minimal number of cliques needed to cover
  $\overline{G}$, i.e. $N(2)$.
\end{proof}

There are several efficient algorithms that compute (approximate)
upper bounds of $N(l)$ (see~\cite{SGHM07}). They are called
\emph{burning} algorithms.  Following a greedy strategy, at every step
they try to select the box that covers (burns) the maximal number of
uncovered (unburned) nodes.  Although they are polynomial algorithms,
we still need to do some further approximations to make the algorithm
of practical use.

First, instead of boxes, we will use \emph{circles}.  
\begin{definition}
  A {\bf circle} of \emph{radius} $r$ and \emph{center} $c$ is a subset of
  nodes of $G$ such that the distance between any of them and the node
  $c$, is smaller that $r$.
\end{definition}

Notice that any circle of radius $r$ is a box of
size $2\,r-1$ (the opposite is in general false) and any box of size
$l$ is a circle of radius $l$ (it does not matter what node of the box we
use as center). Notice also that every radius $r$ and center $c$
characterizes a unique circle. According to Hausdorff's dimension
definition, $N(r) \sim r^{-d}$ also characterizes self-similar graphs
of dimension $d$.

Consider now, a graph $G$ and a radius $r$. At every step, for every
possible node $c$, we could compute the number of unburned nodes
covered by the circle of center $c$ and radius $r$, and select the
node $c$ that maximizes this number, as it is proposed
in~\cite{SGHM07}. However, this algorithm is still too costly for our
purposes. Instead of this, we will apply the following strategy.  We
will order the nodes according to their arity: $\langle
c_1,\dots,c_n\rangle$ such that $arity(c_i) \geq arity (c_j)$, when
$i>j$. Now, for $i=1$ to $n$, if $c_i$ is still unburned and the box
of center $c_i$ and radius $r$ contains some unburned node, select
this circle. Then, we approximate $N(r)$ as the number of selected
circles.

\section{The Fractal Dimension of SAT Instances}\label{sec:dimSAT}
\label{sec:dimensionSAT}

Given a SAT instance, we can build a graph from it. Here, we propose
three models. Given a Boolean formula, the
\emph{Clause-Variable Incidence Graph} (CVIG) associated to it, is a
\emph{bipartite} graph with nodes the set of variables and the set of
clauses, and edges connecting a variable and a clause whenever that
variable occurs in that clause. In the \emph{Variable Incidence Graph}
model (VIG), nodes represent variables, and edges between two nodes
indicate the existence of a clause containing both variables. Finally,
in the \emph{Clause Incidence Graph} model (CIG), nodes represent
clauses, and an edge between two clauses indicates they share a
negated literal. We can define the weighted version of all three
models assigning weights to the edges, such that the sum of the
weights of all edges generated by a clause is equal to one. This way,
we compensate the effect of big clauses $C$ that generate $|C| \choose
2$ edges in the VIG model, and $|C|$ edges in the CVIG model.

In this paper we analyze the function $N(r)$ for the graphs obtained
from a SAT instance following the VIG and CVIG models. These two
functions are denoted $N$ and $N^b$, respectively, and they relate
to each other as follows.
\begin{lemma}\label{lem-NvsNb}
\begin{tabular}[t]{l}
  If $N(r) \sim r^{-d}$ then $N^b(r) \sim
  r^{-d}$.\\
  If $N(r) \sim e^{-\beta\,r}$ then $N^b(r) \sim
  e^{-\frac{\beta}{2}\,r}$.
\end{tabular}
\end{lemma}
\begin{proof}
  Notice that, for any formula, given a circle of radius $r$ in the
  VIG model, using the same center and radius $2\,r-1$ we can cover
  the same variable nodes in the CVIG model. Conversely, given a
  circle of center a clause node $c$ and radius $2\,r+1$ in the CVIG
  model, using an adjacent variable node as center and radius $r+2$ in
  the VIG model, we cover at least the same variable nodes.
  Therefore, we have $N^b(2\,r) \leq N(r)\leq N^ b(2\,r-2)$, and $N(r)
  \sim N^b(2\,r)$. From this asymptotic relation, we can derive the
  two implications stated in the lemma.
\end{proof}

\subsection{Dimension versus Diameter}

The function $N(r)$ determines the \emph{maximal radius} $r^{max}$ of
a graph, defined as the minimum radius of a circle covering the whole
graph.  The maximal radius and the \emph{diameter} $d^{max}$ of a
graph are also related. From these relations we can conclude the
following.

\begin{lemma}\parindent 0mm
  For self-similar graphs or SAT formulas (where $N(r)\sim r^{-d}$),
  the diameter grows polynomially, as $d^{max}\sim n^{1/d} $

  In random graphs or SAT formulas (where $N(r)\sim e^{-\beta\,r}$),
  the diameter grows logarithmically, as $d^{max}\sim\frac{\log
    n}{\beta}$.
\end{lemma}
\begin{proof}
 The diameter of a graph and the maximal radius
are related as $r^{max}\leq d^{max} \leq 2\,r^{max}$.  Notice that $N(1)=n$ is the
number of nodes, and $N(r^{max}) = 1$.  Hence,
\vspace{-2mm}

\hbox to \columnwidth{\hfill
  \begin{tabular}[t]{l}
    if $N(r) =C\,r^{-d}$, then $r^{max}=n^{1/d}$, and\\
    if $N(r) = C\,e^{-\beta\,r}$, then $r^{max}=\frac{\log n}{\beta}+1$. 
  \end{tabular}
\hfill}
\end{proof}

The diameter, as well as the \emph{typical distance} $L$ between
nodes\footnote{The average of the minimal distance between two
  randomly chosen nodes.}, have been widely used in the
characterization of graphs. For instance, \emph{small world graphs}
\cite{Walsh99} are characterized as those graphs with a small typical
distance $L \sim\log n$ and a large clustering coefficient.  This
\emph{definition} of small world graphs is quite imprecise, because it
is difficult to decide what is ``small'' distance and ``large'' coefficient.
Moreover, the diameter (and the typical distance)
of a graph are measures quite expensive to compute in practice, and
very \emph{sensitive}. For instance, given a graph with
$d^{max}\approx\log n$, if we add a chain of $n'\approx c\,n$
connected nodes (a sequence of implications, in the case of a SAT
formula), representing a small fraction of the total number of nodes,
the diameter of the graph grows to $c\,n$.  However, a simple
pre-processing, like unit propagation in the case of SAT formulas, may
destroy this chain and make the diameter drop down again. The typical
distance is a more \emph{stable} measure, however it depends on the
size of the graph.  This means that, to decide if a graph has high
or low typical distance, we have to compare it with the typical
distance in a random graph of the same size.  On the contrary, a
quite good approximation of the fractal dimension can be quickly
computed, and, as it depends on the whole graph, it is quite stable under
simple graph (formula) modifications.  As we will show in this paper,
the fractal dimension of a SAT formula remains quite stable during the
solving process (which involves variable instantiation, and addition
of learnt clauses). Moreover, the dimension is independent of the
size of the graph. Therefore, we advocate for the use of the fractal
dimension instead of the diameter or the typical distance in the
characterization of graphs, search problems or SAT instances.

\subsection{Experimental Evaluation}

We have conducted an exhaustive analysis of the industrial formulas of
the SAT Race 2010 and 2012 Competitions, and some 3CNF random
formulas.

For the random formulas, in the VIG model, we observe that the
function, normalized as $N^{norm}(r) = N(r) / N(1)$ only depends on
the variable/clause ratio (and not on the number of
variables). Moreover, in the phase transition point $m/n =4.25$, the
function has the form $N^{norm}(r) = e^{-2.3\, r}$, i.e. it decays
exponentially with $\beta=2.3$ (see Figure~\ref{fig-random}). Hence,
$r^{max} = \frac{\log n}{2.3}+1$. For instance, for $n=10^6$
variables, random formulas have a radius $r^{max}\approx 7$.  For
bigger values of $m/n$ the decay $\beta$ is bigger.  For values $m/n <
4$ the formula usually forms an unconnected graph, and $N(r)$ is
bigger than the number of partitions. In this case, $N(r)$ decreases
smoothly, even though, it does not seem to have a polynomial $N(r)\sim
r^{-d}$ behavior.  In the CVIG model, we observe the same
behavior. However, in this case, $N(r)$ decays exponentially with
$\beta = 1.2\approx 2.3/2$. Hence, the decay is just half of the decay
of the VIG model, as we expected according to Lemma~\ref{lem-NvsNb}.

Analyzing industrial instances we observe that, in most cases, all
instances of the same family have a very similar normalized function
$N^{norm}(r)$. In Figure~\ref{fig-NdeL}, we show the results for the
\emph{diagnosis} family. We observe a clear heavy tail, typical in
self-similar graphs. We also observe that the functions $N^{norm}(r)$
for ACG-15-10p0 and ACG-20-10p1 (that have similar names) are closer
to each other than to the other instances of the family. The same
happens in other families, like the \emph{bitverif}. Here, three
instances are self-similar, and two not. This suggests that some
families are too heterogeneous, and contain encodings of problems of
different nature. In Figure~\ref{fig-NdeL}, we also show the
results for the \emph{velev} family. In this case, the function $N(r)$
decreases very fast (even faster than for random formulas) and
following an exponential pattern.

\begin{figure}[t]
\hbox to \columnwidth {
  \hss
  \includegraphics[width=0.6\columnwidth]{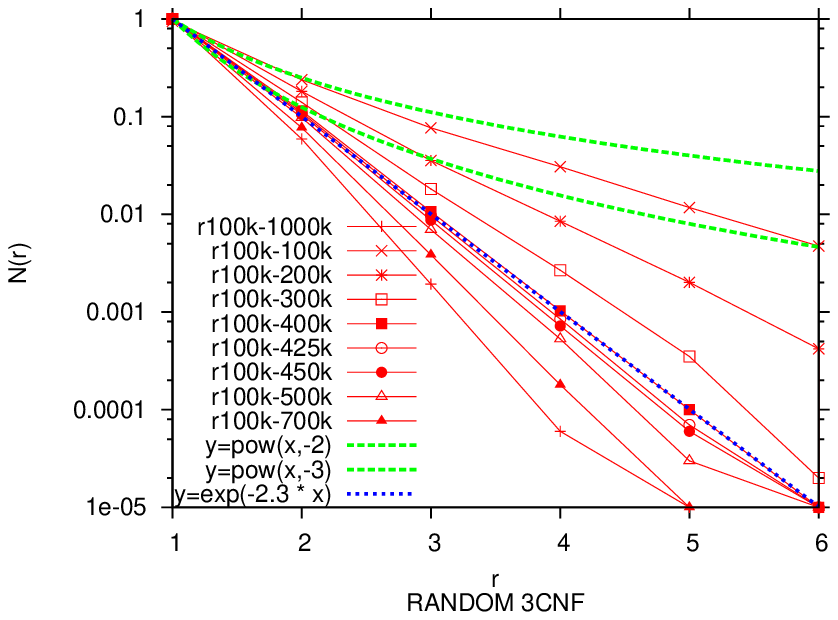}
  \hss 
  \includegraphics[width=0.6\columnwidth]{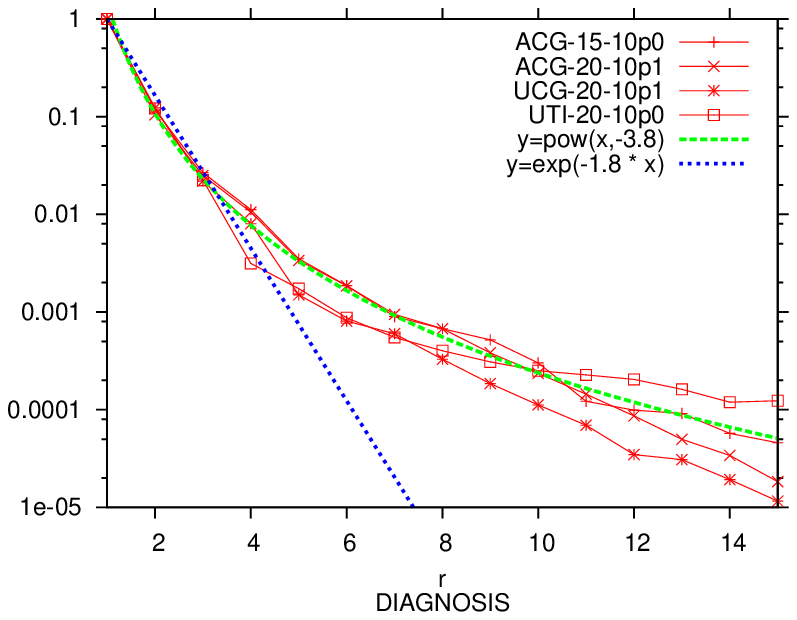}
  \hss
}
\hbox to \textwidth {
  \hss
  \includegraphics[width=0.6\columnwidth]{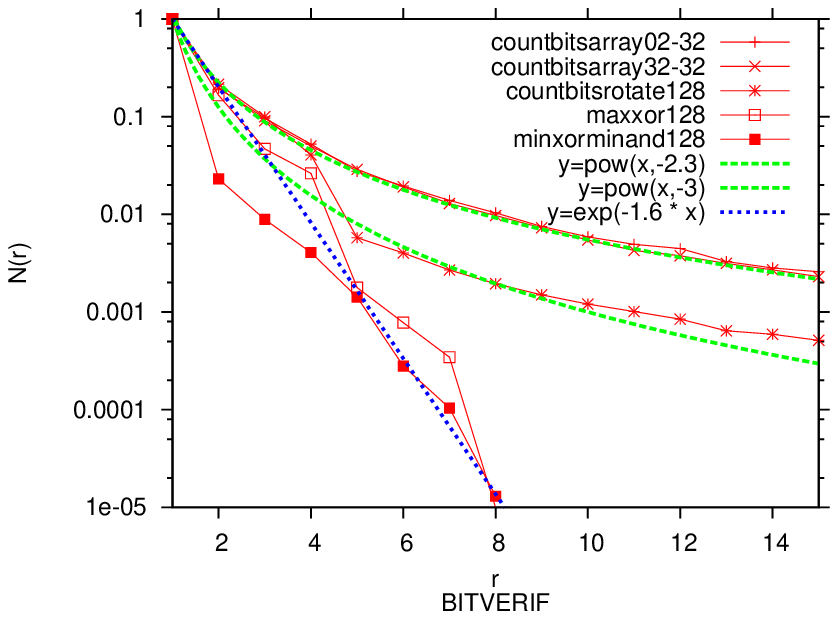}
  \hss 
  \includegraphics[width=0.6\columnwidth]{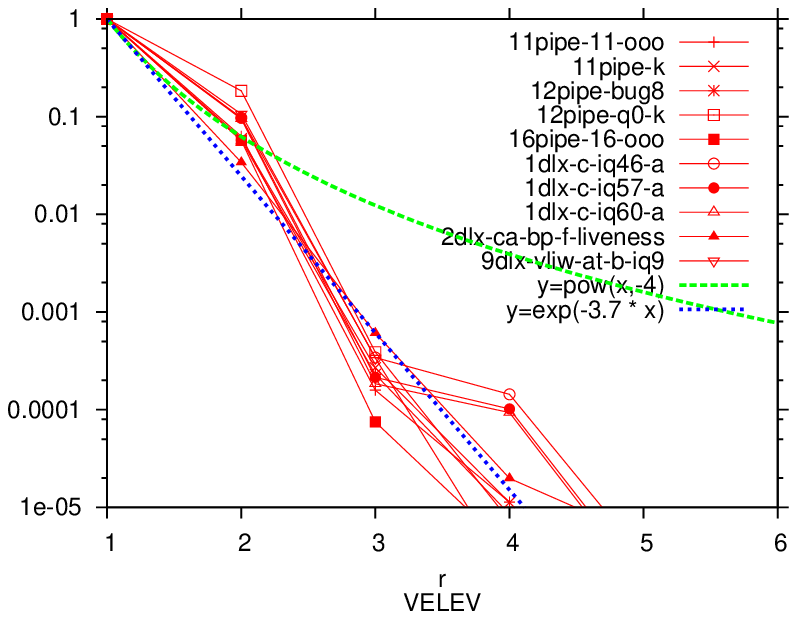}
  \hss 
}
\caption{Function $N(r)$, in the VIG model, for some 3CNF random
  formulas with distinct values of the $n/m$ fraction and some families of formulas of the
  SAT Race 2010. Axes are
  semi-logarithmic, and ``r100k-425k'' indicates $n=10^5$ variables
  and $m=4.25\cdot10^5$ clauses.}
\label{fig-NdeL}\label{fig-random}
\end{figure}

We can conclude that, in the SAT Race 2010 Competition, \emph{velev},
\emph{grieu}, \emph{bioinf} and some \emph{bitverif} instances have a
$N(r)$ function with exponential decay, i.e. are not self-similar;
whereas the rest of instances are all of them self-similar, with
dimensions ranging between $2$ and $3$. In Figure~\ref{fig-family} we
show the dimension of all instances. Since not all formulas are
self-similar, we assign them a \emph{pseudo-dimension} computed as
follows.  If $N(r)\sim r^{-d}$, then $\log N(r) \sim -d\cdot\log r$,
i.e.  the dimension is the slope of a representation of $N(r)$ vs. $r$
using logarithmic axes. Even if $N(r)$ is not polynomial, we compute
the pseudo-dimension as the interpolation, by linear regression, of
$\log N(r)$ vs. $\log r$, using the values for $r=1,\dots,5$.

In Table~\ref{tab:all_dimensions} and Figure~\ref{fig-family}, we present detailed results
of the fractal dimensions, $d$ and $d^b$, and the exponential decays,
$\beta$ and $\beta^b$, of the VIG and CVIG graphs respectively, on the
SAT Race 2010 families and some random instances. These results are
presented using the averages for each family and their standard
deviations. The values we show are computed by linear regression using
as described above.

\subsection{The Effect of Learning}

\begin{figure}[t]
\hbox to \columnwidth {
  \hss\hss
  \includegraphics[width=0.55\columnwidth]{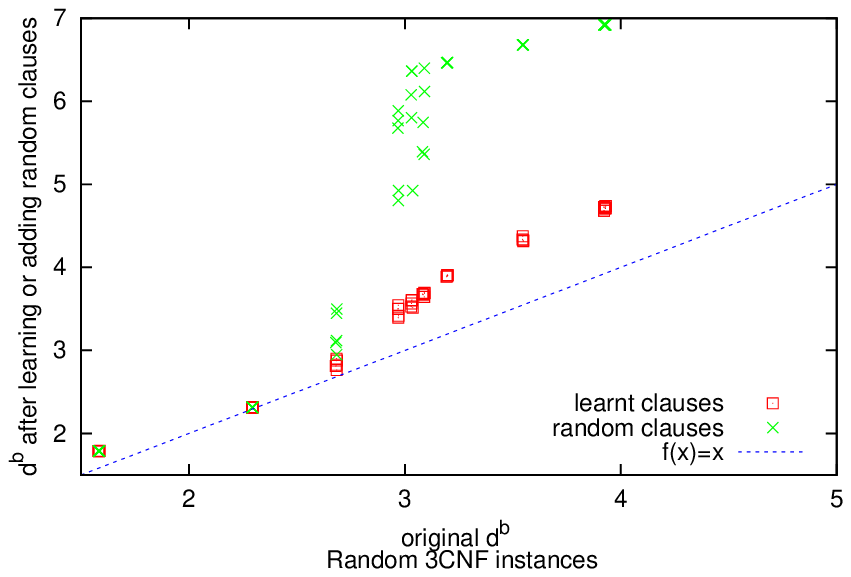}
  \hss 
  \includegraphics[width=0.55\columnwidth]{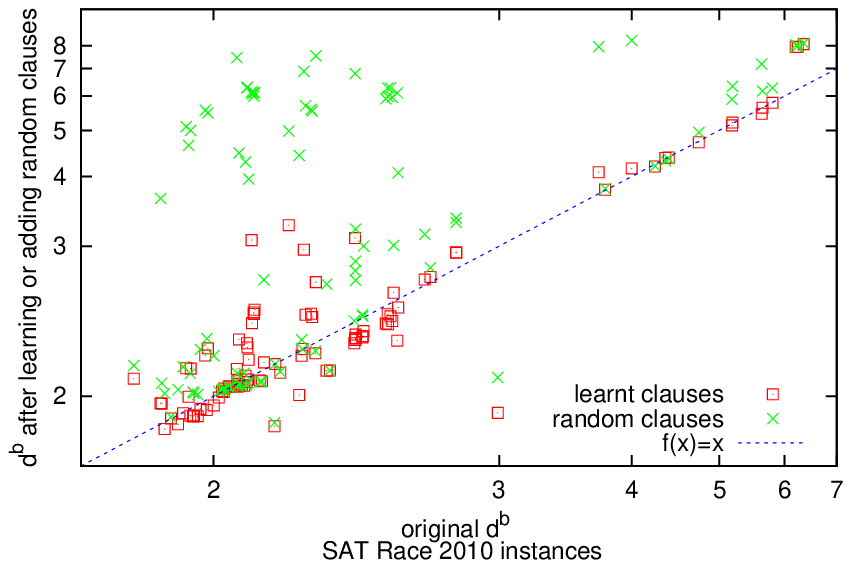}
  \hss\hss
}
\caption{Relation between the original fractal dimension $d^b$, and
  $d^b$ after adding learned clauses (red squares), or after adding random clauses (green crosses),
  at $10^5$ decisions, in random 3CNF formulas (left), and the SAT Race
  2010 instances (right).}
\label{fig-dim-after-decisions}
\end{figure}

\begin{figure}[t]
\hbox to \columnwidth {
  \hss\hss
  \includegraphics[width=0.55\columnwidth]{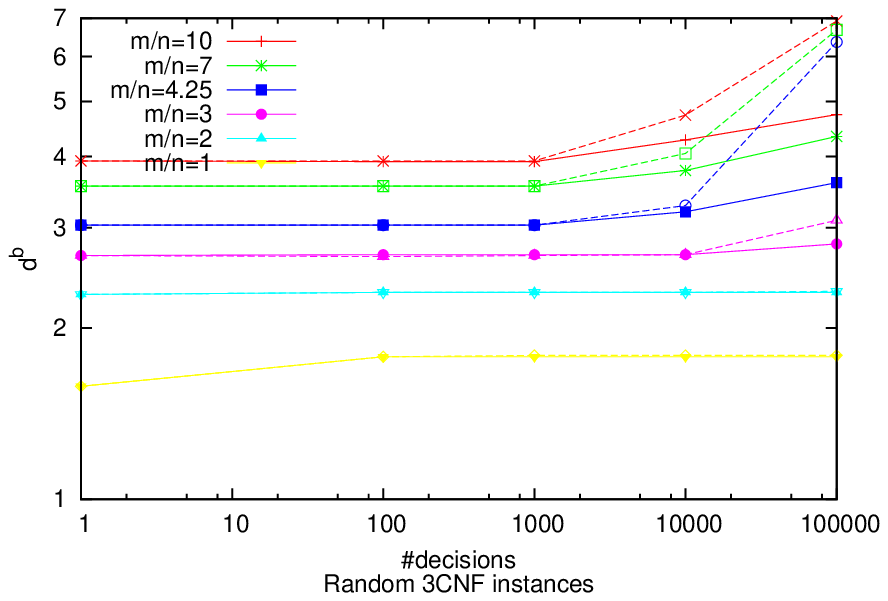}
  \hss 
  \includegraphics[width=0.55\columnwidth]{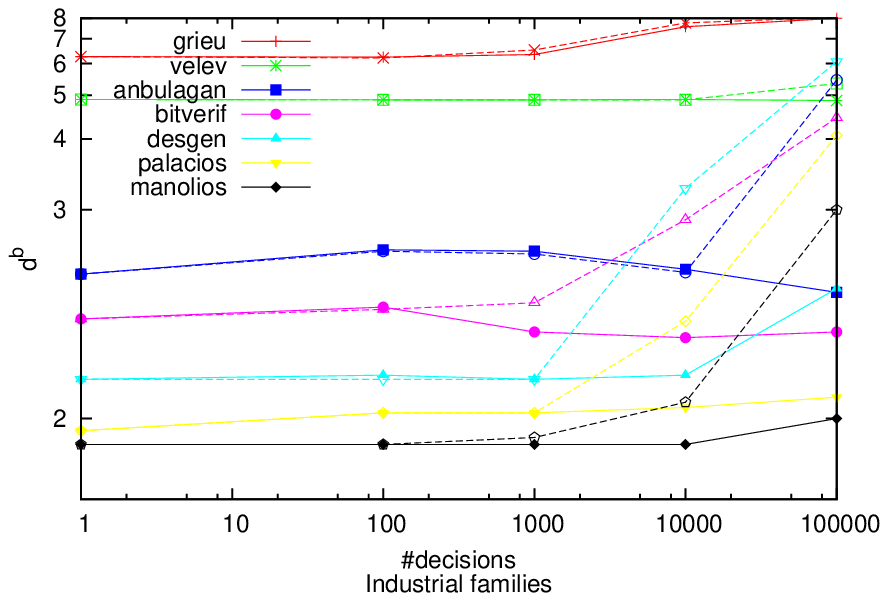}
  \hss\hss
}
\caption{Evolution of the fractal dimension $d^b$ as a function of the number
  of SAT solver decisions, using the learnt clauses (solid line) and random
  clauses (dotted line with same color), for random 3CNF formulas (left) and some industrial families
  (right).}
\label{fig-evolution}
\end{figure}

State-of-the-art SAT solvers, which incorporate Conflict Directed
Clause Learning (CDCL), extend the input formula by adding learnt
clauses from conflicts, during their execution. Unitary learnt clauses
can be propagated without deciding over any variable (i.e., at level 0
of the search tree), simplifying the original formula. Learnt clauses of
bigger length establish (explicitly) new relations between variables.

We have conducted some experiments to analyze how the fractal
dimension evolves during the execution of the SAT solver. First, we
have generated new formulas adding to the original one the learnt
clauses at different depths of the execution (in particular, after
$10^2$, $10^3$, $10^4$ and $10^5$ decisions), and propagating the
unitary clauses.\footnote{In our experimentation, we have used MiniSat
solver \cite{EenSorensson03} since it incorporates CDCL strategies
and it can be easily modified to print learnt clauses.} Then, we
have analyzed the fractal dimensions, $d$ and $d^b$, of these new
formulas. In Tables~\ref{tab:vig_decisions} and~\ref{tab:cvig_decisions} 
and Figures~\ref{fig-dim-after-decisions} and~\ref{fig-evolution}, we present the
values obtained. Columns named $d_x$ and $d^b_x$ represent the fractal
dimensions after $x$ decisions, for the VIG and CVIG respectively.

Two different phenomena can be observed. On one hand, and only in some
cases, after a small number of decisions, the fractal dimension
slightly decreases (see $d_{10^2}$ and $d^b_{10^2}$ in families
\emph{mizh}, \emph{ibm}, \emph{bioinf} and \emph{nec}). This is due to
the learning and propagation of unitary clauses, that simplify the
original formula. Notice that this fact does not happen in random 3CNF
formulas, for which no unitary clauses are learnt. On the other hand,
fractal dimension increases as the execution progresses. This fact is
expected because learnt clauses mean conflicts between subsets of
variables. So, a learnt clause establishes new connections between
variables, directly (in the VIG) or indirectly through nodes
\emph{clause} (in the CVIG). Therefore, the number of tiles needed to
cover the whole graph decreases with the addition of new learnt
clauses, and hence, the fractal dimension becomes higher.  In other
words, new clauses make typical distance decrease, hence fractal
dimension increase.  Empirical results prove this hypothesis in all
the formulas, including random 3CNF.

In a second experiment, we try to quantify this dimension increase.
To do that, we have used the same formulas as before,
but replacing the set of learnt clauses by the same number of random clauses
of the same size. 
In Tables~\ref{tab:vig_decisions} and~\ref{tab:cvig_decisions}
and Figures~\ref{fig-dim-after-decisions} and~\ref{fig-evolution}, we present our
results. Columns named $d_{x-r}$ and $d^b_{x-r}$ represent the fractal
dimensions after $x$ decisions, replacing learnt clauses by random
clauses. In the first steps of execution, the fractal dimension
obtained using random clauses is very close to the values obtained
using the learnt clauses. However, in further steps of the execution,
random clauses produce significantly higher dimension increase than
learnt clauses (except in \emph{velev} and \emph{grieu} families,
where it is very similar).

This 
can be explained as follows: initially, the solver pre-processes the
formula, finding \emph{fast} conflicts and generating \emph{short}
clauses. This is the case of learning and propagating unitary clauses
in some instances. Then, it starts its execution choosing variables
randomly because the activity-based heuristic does not have enough
information to work correctly. This causes the generation of clauses
that connect variables \emph{randomly}, and have the same effect on
the dimension as random clauses.
Once the heuristic starts to work, the solver focuses on subsets of
\emph{local}\footnote{Variables that are very close in the graph.}
variables. While the values of $d_x$ and $d_{x-r}$ are still very
close in random, \emph{velev}, and \emph{grieu} instances (i.e. in the
instances with higher dimension), $d_{x-r}$ is significantly higher
than $d_x$ in the rest of industrial instances (see $d_{10^4}$ for
instance).

We can conclude that CDCL solvers tend to work \emph{locally}, because
conflicts found by the solver concern variables that were already
close in the graph. So these conflicts are useful to explicitly show
some local restrictions, but they hardly ever connect distant parts of
the formula. This strategy seems the most adequate when dealing with
formulas with small dimension (big typical distance between
variables), like most industrial SAT instances.

\section{Additional Graph Properties}\label{sec:other}

In this section we are going to review two other features of CNF
formulas.  These characteristics are also related to the corresponding
graph features, and they are usually studied in the context of
distinguishing random graphs from real networks.


The first feature is the distribution of arities of the nodes in a
graph.  In the classical random graph model~\cite{erdos-renyi}, the
probability that an edge is chosen is constant. Therefore, the node
arities follow a binomial distribution, and most node have about the
same number of edges.  In scale-free graphs node arities follow a
power law distribution $p(k)\sim k^{-\alpha}$, where usually
$2<\alpha<3$. These distributions are characterized by a great
variability.  In recent years it has been observed that many other
real-world graphs, like some social and metabolic networks, also have
a scale-free structure (see~\cite{barabasi99}).

Similarly, in the context of CNF formulas, instances where variables
are selected with a uniform distribution, are called random
formulas. In them, the number of occurrences of a variable also
follows a binomial distribution, and most variables occur about the
same number of times.  In~\cite{CP09}, the distribution of occurrences
of variables in industrial formulas (from the SAT competitions) was
analyzed.  For every instance, they compute the values $f^{real}(k)$,
where $f^{real}(k)$ is the number of variables that have a number of
occurrences equal to $k$.

They see that in many industrial formulas $f^{real}(k)$ is close to a
power-law distribution $k^-\alpha$, where the exponent $\alpha$ range
between $2$ and $3$. 

This value $\alpha$ can be approximated with the
\emph{most-likely method}. As the power-law distribution is intended
to describe only the tail of the distribution, we can discard some
values of $f^{real}(k)$, for small values of $k$. Experimentally, we
observe that in most industrial formulas arities of variables follow a
power-law distribution with $\alpha$ ranging from $2$ to $3$ (see
Figure~\ref{fig-family}).  For the rest of formulas, we compute a
\emph{pseudo-exponent} using the same approximate method, and allowing
to discard up to $5$ values of $f^{real}(k)$, and minimizing the
error, measured as the maximal difference between the real and the
approximated distributions. The computation of the $\alpha$ exponent is
extremely fast.


The second feature to analyze the structure of a SAT instance is the
notion of \emph{modularity} introduced by \cite{newmangirvan04} for
detecting the \emph{community structure} of a graph.  This property is
defined for a graph and a specific \emph{partition} of its vertices
into \emph{communities}, and measures the adequacy of the partition in
the sense that most of the edges are within a community and few of
them connect vertices of distinct communities. The modularity of a
graph is then the maximal modularity for all possible partitions of
its vertices.  Obviously, measured this way, the maximal modularity
would be obtained putting all vertices in the same community. To avoid
this problem, \cite{newmangirvan04} defines modularity as the fraction
of edges connecting vertices of the same community minus the expected
fraction of edges for a random graph with the same number of vertices
and edges.

The problem of maximizing the modularity of a graph is
NP-hard~\cite{brandes06}. As a consequence, most of the
modularity-based algorithms proposed in the literature return an
approximate lower-bound value for the modularity (see a survey
in~\cite{fortunato10}). However, the complexity of many of these
algorithms, make them inadequate for large graphs (as it is the case
of industrial SAT instances, viewed as graphs). For this reason, there
are algorithms specially designed to deal with large-scale networks,
like the greedy algorithms for modularity
optimization~\cite{newman04,clausetetal04}, the label
propagation-based algorithm~\cite{Raghavan} and the method based on
graph folding~\cite{fastmodularity}. 

The community structure of SAT formulas was introduced in~\cite{SAT12}
using the weighted VIG model. 
Here, we
reproduce the analysis for the SAT Race 2010 competition (see
Figure~\ref{fig-family}). We use the folding
algorithm~\cite{fastmodularity}, that relaxing the precision on the
computed approximation, may run in some seconds in most formulas, and
less than $1$ minute in all them.

We could conclude that the \emph{typical} industrial SAT instance is a
formula with a fractal dimension ranging from $2$ to $3$, where
frequencies of variable occurrences follow a power-law distribution
with an exponent also ranging from $2$ to $3$, and a clear community
structure with $Q\approx 0.8$. We think that most SAT solvers are
optimized for dealing with this kind of formulas.

\section{Portfolio SAT Approaches}\label{sec:port}

From the SAT competitions that take place every year since 2002, we
have learnt that no solver dominates over all the instances. From a
theoretical point of view, this makes sense, since the underlying
proof system of SAT solvers is resolution, and it has been shown not
to be automatizable~\cite{AlekhnovichR08} (under strong
assumptions). A proof system is automatizable if there exists an
algorithm that given an unsatisfiable formula, produces a refutation
in time polynomial in the size of shortest
refutation~\cite{BonetPR00}.  Therefore, it seems reasonable to have a
pool of SAT solvers, and given a SAT instance try to predict their
expected running time in order to choose the best candidate.  This is
known as the algorithm selection problem, which consists of choosing
the best algorithm from a predefined set, to run on a problem
instance~\cite{Rice76}.  Algorithm portfolios tackle this problem.

Portfolios have been shown to be very successful in
Satisfiability~\cite{Xu:2008,Kadioglu:2011}, Constraint
Programming~\cite{hydra}, Quantified Boolean
Formulas~\cite{Pulina:2007}, etc. Modern portfolio solvers are an
example of how machine learning can help Constraint
Programming. Machine learning techniques are used to build the
prediction model of the expected running times of a solver on a given
instance.

The first successful algorithm portfolio for SAT was exploited by
SATzilla 2007~\cite{Xu:2008}.  In this algorithm a regression function
is trained to predict the performance of every solver based on the
features of an instance. For a new instance, the solver with the best
predicted runtime is chosen.

The success of modern SAT/CSP solvers is correlated with their ability
to exploit the hidden structure of real-world
instances~\cite{WilliamsGS03}. Therefore, a key element of SAT/CSP
portfolios is to carefully select which features identify the
underlying structure of the instance.  These features correspond to
the attributes the learning algorithm will use to build the classifier
or predictor. The features must be related to the hardness of solving the
instance, since our goal is to predict which solver will be the most
efficient for the given instance. Also the computation has to be
automatizable and with a reasonable cost, since it would not make
sense to consume more time on computing the features than solving the
instance.  For example, in the SATzilla version for the SAT
competition, the timeout for computing the features is around 90
seconds, while the timeout for solving an instances is around 900
seconds in the SAT challenge 2012.

With respect to the features to be analyzed, SATzilla2012 identifies a
total of $138$.  The first $90$ features, introduced for the original
SATzilla can be categorized as follows: problem size features (1-7),
graph based features (8-36), balance features (37-49), proximity to
Horn Formula features (50-55), DPLL probing features (56-62), LP-Based
features (63-68) and local search problem features (69-90). The
features in the last three categories can be expensive to compute in
large instances, and therefore, in practice, we can not use them.  The
rest of the categories, correspond to: clause learning features
(91-108), survey propagation (109-126) and timing (127-138).

As we just mentioned, SATzilla uses graph based features. These
features are extracted from the CVIG, VIG and CIG representations of a
SAT instance as a graph (see Section~\ref{sec:dimensionSAT} for
definitions of these representations). On these graphs, \emph{node
  degree statistics} are computed.  In the case of CVIG, variable and
clause nodes are analyzed independently. Additionally, \emph{diameter
  statistics} and \emph{clustering coefficient statistics} are
computed for the VIG and CIG graphs, respectively. The statistics
involve the computation of the mean, variation coefficient, min, max
and entropy.

\section{Feature-Based SAT Instance Classification}\label{sec:exp}

In order to analyze how good a set of features is for characterizing
SAT instances, we conduct an experimental investigation using
supervised machine learning techniques.  These techniques allow us to
build an instance classifier $h$, that given an instance $\hat{x} =
(x_1, x_2, \ldots, x_m)$ characterized by $m$ computable attributes
(in our case the features of a SAT instance), and a finite set of
class labels $\mathcal{L} = \{\lambda_1, \lambda_2, \ldots,
\lambda_k\}$, decides its label $\hat{\lambda} \in \mathcal{L}$. That
is $h(\hat{x}) = \hat{\lambda}$.  In order to validate the classifier
we use cross-validation. One round of cross-validation involves
partitioning the set of instances into two complementary subsets, the
training set and the test set. The classifier is built with the
training set, while the validation is performed with the test set. In
our experiments, rounds have one instance as test set, and the
rest as training set. We have as many rounds as instances.

Our set of instances comes from the industrial track of the SAT
competitions. Within this track, instances are grouped into families,
according to their industrial application area (e.g. hardware
verification, cryptography, planning, scheduling, etc.). 

We have used the $100$ instances of the SAT Race 2010, that are
grouped into $17$ families. We also tested the $600$ instances of the
SAT Challenge 2012 (application track), that are grouped into $20$
families.  In our experiments, we had to face two problems.  On one
hand, some families are too wide in the sense that the family is not
specific enough. For example, in the \emph{termination} family from
2012, different termination problems are considered, and different
encodings of the same termination problem appear. Notice that having a
different encoding of a problem is enough to alter substantially the
performance of a SAT solver. On the other hand, many formulas are so
hard, that SATzilla features tool crashes computing the features of some of
them. Thus, although our graph properties are computable, it
would make the comparison unfair. Therefore, we decided to focus our
experimentation on the instances from the 2010 competition because the
mentioned problems were fewer. Even in the SAT 2010 set, the problem
to compute the SATzilla features arises, and we had to eliminate the
two instances of the \emph{post} family.

\begin{figure*}[t]
\includegraphics[height=70mm]{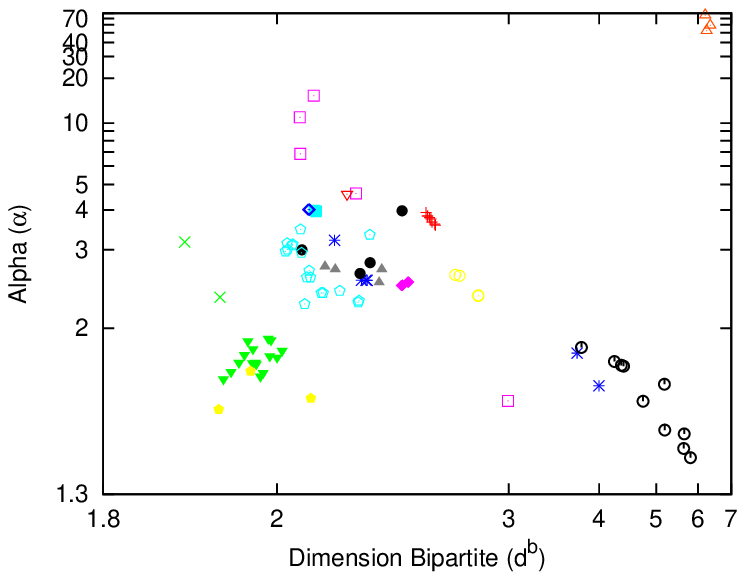}
 
\includegraphics[height=70mm]{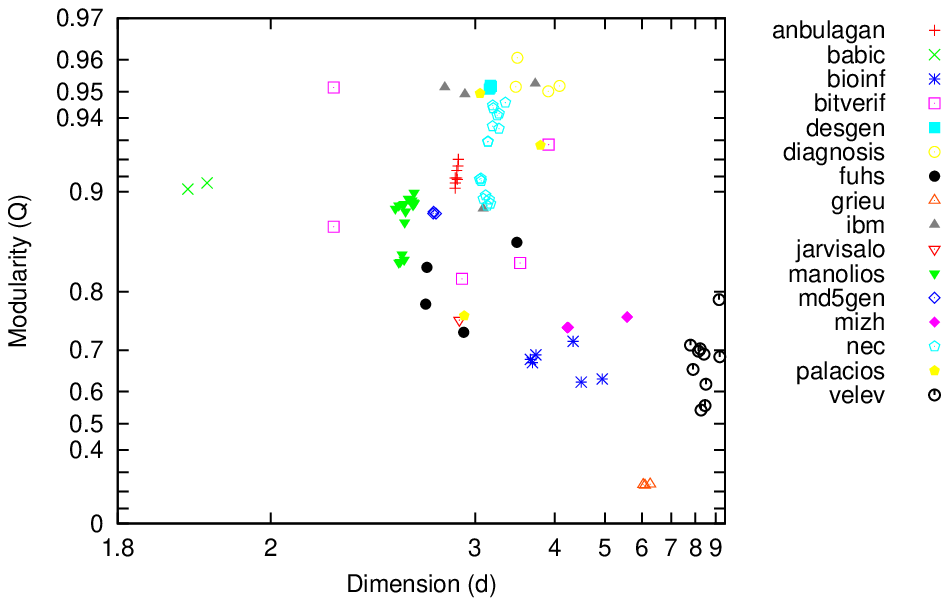}
\caption{Distribution of families according to their $\alpha$ exponent,
  modularity and fractal dimensions.}
\label{fig-family}
\end{figure*}

First we will present the problem of instance classification into
families.  In Figure~\ref{fig-family}, we show the four coordinates of
the graph features: exponent $\alpha$ , modularity $Q$ and fractal
dimensions $d$ and $d^b$, for both VIG and CVIG, respectively.
Instances of each family are plotted with a distinct mark. At first
sight, we can see that instances belonging to the same family are
usually closer to each other, except for the instances of the
\emph{bitverif} family. Thus, we could conclude that most of the
instances are well classified into families by these graph
features. 

In the first experiment we have conducted, we try to validate the
previous hypothesis, doing a cross-validation test on the classifiers
of instances into families. For this purpose, we use the supervised
learning C4.5 algorithm~\cite{Quinlan}. This is a classifying method
based on decision trees. In Table~\ref{tab:all_families}, we present these
results.

\begin{table}[t]
\begin{center}
\begin{tabular}{|c|c|c|c||c|}
\hline
\multicolumn{4}{|c||}{Set of features}		& \# success on 98 inst.\\ \hline
\hline 
\multicolumn{4}{|c||}{ $d^b$, $\alpha$, $m/n$}			& 80 \\ \hline
\multicolumn{4}{|c||}{ $d^b$, $\alpha$, $d$, $Q$, $m/n$}	& 79 \\ \hline
\multicolumn{4}{|c||}{ $d^b$, $\alpha$, $d$, $m/n$}		& 79 \\ \hline
\multicolumn{4}{|c||}{ $d^b$, $\alpha$, $Q$, $m/n$}		& 79 \\ \hline
\multicolumn{4}{|c||}{ $d^b$, $\alpha$, $d$, $Q$}		& 78 \\ \hline
\multicolumn{4}{|c||}{ $\alpha$,  $d$, $Q$}			& 76 \\ \hline
\multicolumn{4}{|c||}{ $\alpha$, $Q$, $m/n$}			& 76 \\ \hline
\multicolumn{4}{|c||}{SATzilla 2012 ($138$ features)}		& 75 \\ \hline
\hline
\end{tabular}
\end{center}
\caption{Impact of the feature set on the success of predicting the
  family a SAT instance belongs to. Features are ordered by their impact in the classifier.}
\label{tab:all_families}
\end{table}

We built two classifiers: one with the $138$ SATzilla features, and
another with the $\alpha$, $Q$, $d$ and $d^b$ features plus the
clause-variable ratio $m/n$. We included $m/n$ into our set of
features since this can be a natural indicator of the hardness of the
instance. As we can see in Table~\ref{tab:all_families}, we obtain
comparable results with respect to the SATzilla-based classifier,
using only $3$ or $4$ features. We tested all the possible subsets,
but we only present those with a success greater or equal to what we
achieve with SATzilla 2012 using the 138 features.  It is important to
notice that the fractal dimension $d^b$ on CVIG appears in the highest
ranked subsets, and seems to be better than using the $d$ on VIG.

Next, in our second experiment we want to check if our features set
could be used as the basis of a portfolio SAT solver.  Thus, we will
use them to predict which is the best SAT solver for a SAT instance.
One of the techniques used in supervised learning is the $k$-NN
($k$-nearest-neighbor) method. It consists on selecting for a test
instance, the classification of the $k$ nearest training instances.
This is the approach used, for instance, in~\cite{MalitskySSS11}.  In
our case, we modify this method as follows. Let $t_i^s$ be the time
needed by solver $s$ on SAT instance $i$, and let $d_{ij}$ be the
distance between test instance $i$ and training instance $j$
(computed using the euclidean distance, according to their
(normalized) feature values).  We can predict the time needed by
solver $s$ on an instance $i$ as
$$
\hat{t}_i^s = \frac{\sum_{j\neq i} t_j^s/d_{i,j}^2}{\sum_{j\neq i} 1/d_{i,j}^2}
$$ 
and choose the solver with a minimal prediction. 

We apply this method to the $5$ features, $\alpha$, $Q$, $d$, $d^b$,
$m/n$, and to the $138$ SATzilla features. These two
\emph{hypothetical portfolios} would use the single SAT solvers of the
SAT Race 2010, and their running times as $t_i^s$\footnote{At \url{http://baldur.iti.uka.de/sat-race-2010/results.html}, the result of SAT Race 2010 are published. Follow this \href{http://baldur.iti.uka.de/sat-race-2010/results/SAT-Race\%202010\%20Results\%20Main\%20Track.html}{link} to find detailed results, including runtimes for all solvers and instances used in this competition.}.  We also simulate
how they would be classified in the same SAT Race 2010. The one based
on SATzilla features would solve $72$ instances (just like the SAT
Race 2010 winner, removing the two \emph{post} instances). The one
based on the four features would solve $71$ instances.

\begin{center}
\tabcolsep 1mm
\begin{tabular}{|l||r|r|r|r|}
\hline
& \small SATzilla  & \small CryptoMiniSat & \small $\alpha$, $Q$, $d$, $d^b$, & \small lingeling \\
& \small features &   & \multicolumn{1}{l|}{$m/n$} & \\
\hline
\#solved & 72 & 72 & 71 & 71 \\
Avg. time & 98.8 & 138.3 & 95.2 & 111.4 \\
\hline
\end{tabular}
\end{center}

These results are still a bit far from the results of the virtual best
solver, that would solve 78 instances. However, if we analyze the
results in each incorrectly classified instance, we can see that there
is not too much room for improvement. For example, one of the
\emph{diagnosis} instances (UTI-20-10p0) is only solved by
CryptoMiniSat. However, this solver does not solve any other of the
instances of the \emph{diagnosis} family. On the contrary, the
lingeling solver (the one both meta-solvers choose) is the best
solving the rest of instances of this family, but does not solve
UTI-20-10p0.  Therefore, any reasonable learning method would fail
selecting a solver for this instance, as both meta-solvers do.

Our final experiment will be to use our features in a state-of-the-art
portfolio. We reached out to the IBM team (winner of some tracks in
2011 \& 2012 SAT-competitions). Their portfolio is based on
hierarchical clustering, conceptually close to decision forests. They
kindly used their portfolio with our 5 features and SATzilla's
138. Not incorporating feature computation time, our feature set
solves 87.2\% instances, and the 138 feature set solves only
82.7\%. Taking into account feature computation time, our features
solve 75.8\% instances, while the 138 feature set solves only 42.85\%.

We cannot explain yet why these features are so much more powerful for
solver selection. However, any classifier is easier to dissect when
based on 5 features rather than 138.

\section{Conclusions}
\label{sec:conclusion}

In this paper we have studied the existence of self-similarity in
industrial SAT instances. We can conclude that, in the SAT Race 2010
Competition, the \emph{velev}, \emph{grieu}, \emph{bioinf} and some
\emph{bitverif} instances are not self-similar; whereas the rest of
instances are all of them self-similar, with fractal dimensions
ranging between $2$ and $3$. These fractal dimensions are very small
when compared with random SAT formulas. Fractal dimension and typical
distances and graph diameter are related (small dimension implies big
distance and diameter).  Hence, industrial SAT instances have a big
diameter (intuitively, we need quite long chains of implications to
propagate a variable instantiation to others).

We have studied the evolution of fractal dimension of SAT formulas
along the execution of a solver.  We can say that, in general, fractal
dimension increases when new learnt clauses are added to the formula,
except in the first steps of solving some industrial instances, where
some unitary clauses are learnt. Moreover, this increase is specially
abrupt in those instances that show exponential decays (for instance,
in the family \emph{grieu} or random formulas).  This increase is
small, if we compare it with the effect of adding random clauses.
Therefore, learning \emph{does not} contribute very much to connect
distant parts of the formula, as one could think.

We have explored how these graph features plus the clause-variable
ratio could be used within portfolios to characterize SAT
instances. First, we observed that these five features can be used to
classify SAT instances into families comparing favorably to the
results obtained with the $138$ features from SATzilla2012. Second, we
simulated how two hypothetical portfolios would have performed in the
SAT Race 2010 Competition using the four features, and the $138$
features from SATzilla2012, respectively. We observed that they
perform similarly. Third, we provided data from a real portfolio that
shows the effectiveness of this approach.

As future work, we plan to investigate into more detail how to use
structural graph features such as the fractal dimension, the $\alpha$
exponent or the modularity, to design more efficient single SAT
solvers.

\bibliography{corr13}

\newpage

\section*{Appendix}

In this appendix we present all the numerical values that have been
used in the plotting of the figures.

\begin{table}[h]
\begin{center}
\begin{tabular}{|c|l||r|r||r|r|}\hline

\multicolumn{2}{|c||}{Family}& 
\multicolumn{2}{c||}{Variable IG}& 
\multicolumn{2}{c|}{Clause-Variable IG} \\

\multicolumn{2}{|c||}{(\#instanc.)} & 
$d$ & $\beta$ &
$d^b$ & $\beta^b$ \\\hline\hline
\multirow{3}{*}{\begin{sideways}cripto.\end{sideways}}
 & desgen (4)&   3.15 $\pm$ 0.00  & 1.28 $\pm$ 0.00 &  
                         2.08 $\pm$ 0.00  & 0.82 $\pm$ 0.00  \\
 & md5gen (3) & 2.67 $\pm$ 0.00 & 1.08 $\pm$ 0.00 &
               2.07 $\pm$ 0.00 & 0.81 $\pm$ 0.00 \\
 & mizh (8) & 4.73 $\pm$ 0.42 & 2.05 $\pm$ 0.11 &
               2.37 $\pm$ 0.00 & 0.95 $\pm$ 0.00 \\\hline
\multirow{3}{*}{\begin{sideways}hard. ver.\end{sideways}}
 & ibm (4) &  3.11 $\pm$ 0.13 & 1.24 $\pm$ 0.03 &  
               2.20 $\pm$ 0.00 & 0.86 $\pm$ 0.00  \\
 & manolios (16) & 2.49 $\pm$ 0.00 & 0.98 $\pm$ 0.00 &
               1.96 $\pm$ 0.00 & 0.77 $\pm$ 0.00 \\
 & velev (10) & 8.41 $\pm$ 0.20 & 3.82 $\pm$ 0.23 &
               4.89 $\pm$ 0.44 & 1.93 $\pm$ 0.07 \\\hline
\multirow{6}{*}{\begin{sideways}mixed\end{sideways}}
 & anbulagan (8) & 2.84 $\pm$ 0.00 & 1.13 $\pm$ 0.00 &
               2.49 $\pm$ 0.00 & 0.98 $\pm$ 0.00 \\
 & bioinf (6) & 4.12 $\pm$ 0.24 & 1.83 $\pm$ 0.16 &
              2.75 $\pm$ 0.61 & 1.04 $\pm$ 0.07 \\
 & diagnosis (4) & 3.72 $\pm$ 0.08 & 1.49 $\pm$ 0.02 &
               2.71 $\pm$ 0.00 & 1.07 $\pm$ 0.00 \\
 & grieu (3) & 6.13 $\pm$ 0.01 & 3.35 $\pm$ 0.00 &
               6.26 $\pm$ 0.00 & 2.42 $\pm$ 0.00 \\
 & jarvisalo (1) & 2.86 & 1.19 &
               2.17 & 0.85 \\
 & palacios (3) & 3.24 $\pm$ 0.15 & 1.25 $\pm$ 0.02 &
               1.98 $\pm$ 0.01 & 0.77 $\pm$ 0.00 \\\hline
\multirow{4}{*}{\begin{sideways}soft. ver.\end{sideways}}
 & babic (2) &  1.71 $\pm$ 0.03 & 0.67 $\pm$ 0.01 &
               1.89 $\pm$ 0.00 & 0.73 $\pm$ 0.00 \\
 & bitverif (5) & 2.92 $\pm$ 0.48 & 1.17 $\pm$ 0.08 &
              2.27 $\pm$ 0.13 & 0.90 $\pm$ 0.02 \\
 & fuhs (4) & 2.90 $\pm$ 0.12 & 1.18 $\pm$ 0.03 &
               2.22 $\pm$ 0.01 & 0.87 $\pm$ 0.00 \\
 & nec (17) & 3.15 $\pm$ 0.01 & 1.28 $\pm$ 0.00 &
               2.09 $\pm$ 0.01 & 0.82 $\pm$ 0.00 \\
\hline\hline
\multicolumn{2}{|l||}{$rand_{m/n=1.00}$ (5)} &
1.77 $\pm$ 0.00 & 0.69 $\pm$ 0.00 & 1.58 $\pm$ 0.00 & 0.64 $\pm$ 0.00\\
\multicolumn{2}{|l||}{$rand_{m/n=2.00}$ (5)} &
3.40 $\pm$ 0.00 & 1.37 $\pm$ 0.00 & 2.29 $\pm$ 0.00 & 0.93 $\pm$ 0.00\\
\multicolumn{2}{|l||}{$rand_{m/n=3.00}$ (5)} &
4.67 $\pm$ 0.00 & 1.92 $\pm$ 0.00 & 2.68 $\pm$ 0.00 & 1.09 $\pm$ 0.00 \\
\multicolumn{2}{|l||}{$rand_{m/n=4.00}$ (5)} &
5.54 $\pm$ 0.00 & 2.29 $\pm$ 0.00 & 2.97 $\pm$ 0.00 & 1.21 $\pm$ 0.00 \\
\multicolumn{2}{|l||}{$rand_{m/n=4.25}$ (5)} &
5.76 $\pm$ 0.00 & 2.39 $\pm$ 0.00 & 3.03 $\pm$ 0.00 & 1.23 $\pm$ 0.00 \\
\multicolumn{2}{|l||}{$rand_{m/n=4.50}$ (5)} &
5.97 $\pm$ 0.02 & 2.48 $\pm$ 0.00 & 3.09 $\pm$ 0.00 & 1.26 $\pm$ 0.00 \\
\multicolumn{2}{|l||}{$rand_{m/n=5.00}$ (5)} &
6.49 $\pm$ 0.05 & 2.72 $\pm$ 0.01 & 3.20 $\pm$ 0.00 & 1.30 $\pm$ 0.00 \\
\multicolumn{2}{|l||}{$rand_{m/n=7.00}$ (5)} &
7.01 $\pm$ 0.00 & 2.91 $\pm$ 0.00 & 3.55 $\pm$ 0.00 & 1.44 $\pm$ 0.00 \\
\multicolumn{2}{|l||}{$rand_{m/n=10.00}$(5)} &
7.35 $\pm$ 0.00 & 3.01 $\pm$ 0.00 & 3.93 $\pm$ 0.00 & 1.58 $\pm$ 0.00 \\\hline
\end{tabular}
\end{center}
\caption{Fractal dimensions, $d$ and $d^b$, and exponential decays, $\beta$ and
$\beta^b$, of the VIG and CVIG respectively, of the families of the
SAT Race 2010 and some random formulas (of $n=10^5$ variables and
$m$ clauses). Values are presented by averages for families
and their standard deviations.}
\label{tab:all_dimensions}
\end{table}

\begin{table}[t]
\begin{center}
\begin{tabular}{|c|l||r||r|r|r|r||r|r|r|r|}\hline

\multicolumn{2}{|c||}{Family}& 
\multicolumn{9}{c|}{Variable IG} \\

\multicolumn{2}{|c||}{(\#instanc.)} & 
$d_{orig}$ & $d_{10^2}$ & $d_{10^3}$ & $d_{10^4}$ & $d_{10^5}$ &
$d_{10^2-r}$ & $d_{10^3-r}$ & $d_{10^4-r}$ &
$d_{10^5-r}$ \\ \hline\hline
\multirow{3}{*}{\begin{sideways}cripto.\end{sideways}}
 & desgen(4)&  3.15& 3.16& 3.19{\tiny (4)}& 3.33{\tiny (4)}& 3.53{\tiny (4)}&
                         3.13& 3.16{\tiny (4)} &  6.88{\tiny (4)} & 9.31{\tiny (4)} \\
 & md5gen(3) & 2.67& 2.71& 2.71{\tiny (3)}& 2.78{\tiny (3)}& 2.87{\tiny (2)} &
                          2.72 & 2.71{\tiny (3)} & 7.34{\tiny (3)} & 10.27{\tiny (2)}\\
 & mizh(8) & 4.73& 2.96 & 2.95{\tiny (8)} & 2.95{\tiny (8)} & 2.94{\tiny (8)} &
                     2.95 & 2.94{\tiny (8)} & 2.91{\tiny (8)} & 4.77{\tiny (8)}\\\hline
\multirow{3}{*}{\begin{sideways}hard. ver.\end{sideways}}
 & ibm(4) &  3.11 & 3.03 &  3.11{\tiny (3)} & 3.03{\tiny (3)} & 2.70{\tiny (2)}& 
                    3.01 & 3.05{\tiny (3)} &  3.00{\tiny (3)} &  4.65{\tiny (2)}  \\
 & manolios(16) & 2.49& 2.50 & 2.49{\tiny (16)} & 2.54{\tiny (16)} & 2.85{\tiny (14)} & 
                             2.50 & 2.51{\tiny (16)} & 2.95{\tiny (16)} & 4.91{\tiny (14)} \\
 & velev(10) & 8.41& 8.42 & 8.42{\tiny (10)} & 8.43{\tiny (10)} & 8.78{\tiny (10)} & 
                       8.39 & 8.39{\tiny (10)} & 8.39{\tiny (10)} & 8.75{\tiny (10)}\\\hline
\multirow{6}{*}{\begin{sideways}mixed\end{sideways}}
 & anbulagan(8) & 2.84& 2.79 & 2.79{\tiny (8)} & 2.77{\tiny (8)} & 2.91{\tiny (8)} &
                             2.79 & 2.83{\tiny (8)} & 3.50{\tiny (8)} & 9.04{\tiny (8)}\\
 & bioinf(6) &  4.12& 4.05 & 3.76{\tiny (6)} & 4.11{\tiny (5)} & 4.25{\tiny (5)} &
                       3.87 & 4.38{\tiny (6)} & 6.16{\tiny (5)} & 9.81{\tiny (5)}\\
 & diagnosis(4) & 3.72& 3.62 & 3.62{\tiny (4)} & 3.53{\tiny (4)} & 3.46{\tiny (4)}&
                            3.59 & 3.60{\tiny (4)} & 3.77{\tiny (4)} & 7.51{\tiny (4)} \\
 & grieu(3) & 6.13& 6.13 & 6.16{\tiny (3)} & 8.53{\tiny (3)} & 9.68{\tiny (3)}&
                     6.12 & 6.21{\tiny (3)} & 8.55{\tiny (3)} & 9.68{\tiny (3)}\\
 & jarvisalo(1) & 2.86& 3.04& 4.72{\tiny (1)}& 6.39{\tiny (1)}& 5.73{\tiny (1)}&
                          4.34& 6.77{\tiny (1)} & 7.01{\tiny (1)} & 6.97{\tiny (1)}\\
 & palacios(3) & 3.24& 6.63 & 6.62{\tiny (3)} & 6.39{\tiny (3)} & 6.50{\tiny (3)} &
                          6.63 & 6.62{\tiny (3)} & 7.09{\tiny (3)} & 7.90{\tiny (3)} \\\hline
\multirow{5}{*}{\begin{sideways}soft. ver.\end{sideways}}
 & babic(2) &  1.71& 2.20 & 2.20{\tiny (2)} & 2.23{\tiny (2)} & 2.26{\tiny (2)} & 
                       2.17 & 2.19{\tiny (2)} & 2.24{\tiny (2)} & 2.56{\tiny (2)}\\
 & bitverif(5) & 2.92& 2.96 & 3.08{\tiny (5)} & 3.64{\tiny (5)} & 4.28{\tiny (5)} &
                        2.96 & 3.52{\tiny (5)} & 5.04{\tiny (5)} & 6.81{\tiny (5)}\\
 & fuhs(4) & 2.90& 2.89 & 3.09{\tiny (4)} & 3.48{\tiny (4)} & 3.73{\tiny (4)} &
                    2.88 & 3.91{\tiny (4)} & 6.32{\tiny (4)} & 8.46{\tiny (4)}\\
 & nec(17) & 3.15 & 2.93 & 2.91{\tiny (14)} & 2.89{\tiny (2)} & - {\tiny (0)} &
                    2.94 & 2.89{\tiny (14)} & 3.19{\tiny (2)} & - {\tiny (0)}\\\hline
\hline
\multicolumn{2}{|l||}{$rand_{\alpha=1.00}$ (5)} & 1.77 & 2.67 & 2.67{\tiny (5 )}& 2.67{\tiny (5)} & 2.67{\tiny (5)} &
	2.66 & 2.67{\tiny (5)} & 2.66{\tiny (5)} & 2.68{\tiny (5)} \\
\multicolumn{2}{|l||}{$rand_{\alpha=2.00}$ (5)} & 3.40 & 3.79& 3.79{\tiny (5)} & 3.79{\tiny (5)} & 3.81{\tiny (5)} & 
	3.79 & 3.80{\tiny (5)} & 3.79{\tiny (5)} & 3.79{\tiny (5)} \\
\multicolumn{2}{|l||}{$rand_{\alpha=3.00}$ (5)} & 4.67 & 4.78 & 4.78{\tiny (5)} & 4.78{\tiny (5)} & 7.22{\tiny (5)} &
	4.78 & 4.76{\tiny (5)} & 4.78{\tiny (5)} & 7.42{\tiny (5)} \\
\multicolumn{2}{|l||}{$rand_{\alpha=4.00}$ (5)} & 5.54 & 5.60 & 5.60{\tiny (5)} & 7.27{\tiny (5)} & 7.65{\tiny (5)} &
	5.60 & 5.57{\tiny (5)} & 7.33{\tiny (5)} & 9.99{\tiny (5)} \\
\multicolumn{2}{|l||}{$rand_{\alpha=4.25}$ (5)} & 5.76 & 5.85 & 5.85{\tiny (5)} & 7.18{\tiny (5)} & 7.75{\tiny (5)} &
	5.82 & 5.77{\tiny (5)} & 7.21{\tiny (5)} & 10.35{\tiny (5)} \\
\multicolumn{2}{|l||}{$rand_{\alpha=4.50}$ (5)} & 5.97 & 5.98 & 5.98{\tiny (5)} & 7.47{\tiny (5)} & 7.82{\tiny (5)} &
	6.00 & 6.00{\tiny (5)} & 7.48{\tiny (5)} & 10.05{\tiny (5)} \\
\multicolumn{2}{|l||}{$rand_{\alpha=5.00}$ (5)} & 6.49 & 6.49 & 6.49{\tiny (5)} & 7.49{\tiny (5)} & 8.07{\tiny (5)} &
	6.49 & 6.48{\tiny (5)} & 7.53{\tiny (5)} & 10.52{\tiny (5)} \\
\multicolumn{2}{|l||}{$rand_{\alpha=7.00}$ (5)} & 7.01 & 7.02 & 7.02{\tiny (5)} & 7.68{\tiny (5)} & 9.03{\tiny (5)} &
	7.02 & 7.02{\tiny (5)} & 7.70{\tiny (5)} & 10.53{\tiny (5)} \\
\multicolumn{2}{|l||}{$rand_{\alpha=10.00}$(5)} & 7.35 & 7.35 & 7.35{\tiny (5)} & 7.93{\tiny (5)} & 9.66{\tiny (5)} &
	7.34 & 7.34{\tiny (5)} & 8.02{\tiny (5)} & 10.60{\tiny (5)} \\ 
\hline
\end{tabular}
\end{center}
\caption{Evolution of the fractal dimension $d$ of SAT Race 2010 and some
random formulas using VIG. $d_{orig}$ stands for the fractal dimension
of the original formula. $d_x$ stands for the fractal dimension of the
new formula generated adding the learnt clauses after $x$ decisions to
the original formula. $d_{x-rand}$ stands for the fractal dimension of
a formula generated adding to the original formula as random clauses
as learnt clauses, and of the same size. Numbers in brackets represent
the number of instance that are not still solved as UNSAT at that
depth.}
\label{tab:vig_decisions}
\end{table}

\begin{table}[t]
\begin{center}
\begin{tabular}{|c|l||r||r|r|r|r||r|r|r|r|r|}\hline
\multicolumn{2}{|c||}{Family}& 
\multicolumn{9}{c|}{Clause-Variable IG} \\
\multicolumn{2}{|c||}{(\#instanc.)} & 
$d^b_{orig}$ & $d^b_{10^2}$ & $d^b_{10^3}$ & $d^b_{10^4}$ &
$d^b_{10^5}$ & $d^b_{10^2-r}$ & $d^b_{10^3-r}$ &
$d^b_{10^4-r}$ & $d^b_{10^5-r}$ \\ \hline\hline
\multirow{3}{*}{\begin{sideways}cripto.\end{sideways}}
 & desgen(4)& 2.08&  2.09 & 2.08{\tiny (4)}&  2.09{\tiny (4)} & 2.41{\tiny (4)} &
		2.08&  2.08{\tiny (4)} & 3.24{\tiny (4)} &  6.08{\tiny (4)}\\
 & md5gen(3)& 2.07& 2.07 & 2.07{\tiny (3)} & 2.08{\tiny (3)} & 2.23{\tiny (2)} &
	 2.07& 2.07{\tiny (3)} & 3.84{\tiny (3)} & 6.29{\tiny (2)}\\
 & mizh(8) & 2.37& 2.21 & 2.22{\tiny (8)} & 2.22{\tiny (8)} & 2.27{\tiny (8)} &
	2.21 & 2.22{\tiny (8)} & 2.22{\tiny (8)} & 2.70{\tiny (8)}\\\hline
\multirow{3}{*}{\begin{sideways}hard. ver.\end{sideways}}
 & ibm(4) &  2.20& 2.11 & 2.12{\tiny (3)} &  2.12{\tiny (3)} & 2.12{\tiny (2)} &
	2.11 & 2.09{\tiny (3)} &  2.11{\tiny (3)} & 2.65{\tiny (2)}\\
 & manolios(16) & 1.96 & 1.96 & 1.96{\tiny (16)} & 1.96{\tiny (16)} & 2.00{\tiny (14)} &
	1.96 & 1.97{\tiny (16)} & 2.03{\tiny (16)} & 3.00{\tiny (14)}\\
 & velev(10) & 4.89 & 4.88 & 4.88{\tiny (10)} & 4.89{\tiny (10)} & 4.86{\tiny (10)} &
	4.88 & 4.88{\tiny (10)} & 4.88{\tiny (10)} & 5.34{\tiny (10)}\\\hline
\multirow{6}{*}{\begin{sideways}mixed\end{sideways}}
 & anbulagan(8) & 2.49 & 2.65 & 2.64{\tiny (8)} & 2.52{\tiny (8)} & 2.39{\tiny (8)} &
	2.64 & 2.62{\tiny (8)} & 2.50{\tiny (8)} & 5.45{\tiny (8)}\\
 & bioinf(6) & 2.75 & 2.67 & 2.63{\tiny (6)} & 2.80{\tiny (5)} & 3.09{\tiny (5)} &
	2.68 & 3.37{\tiny (6)} & 4.59{\tiny (5)} & 6.60{\tiny (5)}\\
 & diagnosis(4) & 2.71& 2.82 & 2.81{\tiny (4)} & 2.81{\tiny (4)} & 2.81{\tiny (4)} &
	2.79& 2.76{\tiny (4)} & 2.75{\tiny (4)} & 3.14{\tiny (4)}\\
 & grieu(3) & 6.26& 6.24 & 6.34{\tiny (3)} & 7.56{\tiny (3)} & 8.00{\tiny (3)} &
	6.21 & 6.52{\tiny (3)} & 7.76{\tiny (3)} & 8.05{\tiny (3)}\\
 & jarvisalo(1) & 2.17& 2.19& 2.40{\tiny (1)} & 3.32{\tiny (1)} & 3.23{\tiny (1)} &
	2.34& 4.66{\tiny (1)} & 4.96{\tiny (1)} & 4.98{\tiny (1)}\\
 & palacios(3) & 1.98 & 2.01 & 2.01{\tiny (3)} & 2.02{\tiny (3)} & 2.04{\tiny (3)} &
	2.01 & 2.01{\tiny (3)} & 2.26{\tiny (3)} & 4.07{\tiny (3)}\\\hline
\multirow{5}{*}{\begin{sideways}soft. ver.\end{sideways}}
 & babic(2) & 1.89 & 2.02 & 2.02{\tiny (2)} & 2.02{\tiny (2)} & 2.02{\tiny (2)} &
	2.02 & 2.02{\tiny (2)} & 2.03{\tiny (2)} & 2.09{\tiny (2)}\\
 & bitverif(5) & 2.27 & 2.32 & 2.22{\tiny (5)} & 2.20{\tiny (5)} & 2.22{\tiny (5)} &
	2.31 & 2.34{\tiny (5)} & 2.90{\tiny (5)} & 4.44{\tiny (5)}\\
 & fuhs(4) & 2.22 & 2.20 & 2.23{\tiny (4)} & 2.40{\tiny (4)} & 2.74{\tiny (4)} &
	2.19 & 2.53{\tiny (4)} & 4.21{\tiny (4)} & 6.43{\tiny (4)}\\
 & nec(17) & 2.09 & 2.07 & 2.08{\tiny (14)} & 2.11{\tiny (2)} & - {\tiny (0)} &
	2.07 & 2.08{\tiny (14)} & 2.17{\tiny (2)} & - {\tiny (0)}\\\hline
\hline
\multicolumn{2}{|l||}{$rand_{\alpha=1.00}$ (5)} & 1.58 & 1.79 & 1.79{\tiny (5)} & 1.79{\tiny (5)} & 1.79{\tiny (5)} &
	1.79& 1.79{\tiny (5)} & 1.79{\tiny (5)} & 1.79{\tiny (5)} \\
\multicolumn{2}{|l||}{$rand_{\alpha=2.00}$ (5)} & 2.29 & 2.31 & 2.31{\tiny (5)} & 2.31{\tiny (5)} & 2.31{\tiny (5)} &
	2.31 & 2.31{\tiny (5)} & 2.31{\tiny (5)} & 2.31{\tiny (5)} \\
\multicolumn{2}{|l||}{$rand_{\alpha=3.00}$ (5)} & 2.68 & 2.68& 2.68{\tiny (5)} & 2.69{\tiny (5)} & 2.84{\tiny (5)} &
	2.68& 2.69{\tiny (5)} & 2.68{\tiny (5)} & 3.22{\tiny (5)} \\
\multicolumn{2}{|l||}{$rand_{\alpha=4.00}$ (5)} & 2.97& 2.97 & 2.97{\tiny (5)} & 3.11{\tiny (5)} & 3.47{\tiny (5)} &
	2.97 & 2.97{\tiny (5)} & 3.18{\tiny (5)} & 5.41{\tiny (5)} \\
\multicolumn{2}{|l||}{$rand_{\alpha=4.25}$ (5)} & 3.03 & 3.03 & 3.03{\tiny (5)} & 3.17{\tiny (5)} & 3.56{\tiny (5)} &
	3.03 & 3.03{\tiny (5)} & 3.25{\tiny (5)} & 5.91{\tiny (5)} \\
\multicolumn{2}{|l||}{$rand_{\alpha=4.50}$ (5)} & 3.09 & 3.09& 3.09{\tiny (5)} & 3.26{\tiny (5)} & 3.68{\tiny (5)} &
	3.09& 3.09{\tiny (5)} & 3.34{\tiny (5)} & 5.80{\tiny (5)} \\
\multicolumn{2}{|l||}{$rand_{\alpha=5.00}$ (5)} & 3.20 & 3.20 & 3.20{\tiny (5)} & 3.39{\tiny (5)} & 3.90{\tiny (5)} &
	3.20 & 3.20{\tiny (5)} & 3.48{\tiny (5)} & 6.47{\tiny (5)} \\
\multicolumn{2}{|l||}{$rand_{\alpha=7.00}$ (5)} & 3.55& 3.55 & 3.55{\tiny (5)} & 3.78{\tiny (5)} & 4.34{\tiny (5)} &
	3.55 & 3.55{\tiny (5)} & 4.00{\tiny (5)} & 6.68{\tiny (5)} \\
\multicolumn{2}{|l||}{$rand_{\alpha=10.00}$(5)} & 3.93 & 3.93 & 3.93{\tiny (5)} & 4.27{\tiny (5)} & 4.71{\tiny (5)} &
	3.93 & 3.93{\tiny (5)} & 4.70{\tiny (5)} & 6.92{\tiny (5)} \\
\hline
\end{tabular}
\end{center}
\caption{Evolution of the fractal dimension $d^b$ of SAT Race 2010 and some
random formulas using CVIG. $d^b_{orig}$ stands for the fractal
dimension of the original formula. $d^b_x$ stands for the fractal
dimension of the new formula generated adding the learnt clauses after
$x$ decisions to the original formula. $d^b_{x-rand}$ stands for the
fractal dimension of a formula generated adding to the original
formula as random clauses as learnt clauses, and of the same
size. Numbers in brackets represent the number of instance that are
not still solved as UNSAT at that depth.}\label{tab:cvig_decisions}
\end{table}


\end{document}